%% file: main.tex
\documentclass[sigconf]{acmart}

\usepackage[utf8]{inputenc} 
\usepackage[T1]{fontenc}    
\usepackage{hyperref}       
\usepackage{url}            
\usepackage{booktabs}       
\usepackage{amsfonts}       
\usepackage{nicefrac}       
\usepackage{microtype}      
\usepackage{xcolor}         

\usepackage{graphicx}
\usepackage{subfigure}
\usepackage{bm}
\usepackage{enumitem}
\usepackage{amsmath}
\usepackage{amsthm}
\usepackage{multicol}
\usepackage{multirow}
\usepackage{algorithm}
\usepackage{algorithmic}

\usepackage{amssymb}
\usepackage{wasysym}

\usepackage{colortbl}

\usepackage{wrapfig}

\newcommand{\paratitle}[1]{\vspace{1.5ex}\noindent\textbf{#1}}

\newtheorem{theorem}{Theorem}

\AtBeginDocument{%
  \providecommand\BibTeX{{%
    \normalfont B\kern-0.5em{\scshape i\kern-0.25em b}\kern-0.8em\TeX}}}

\setcopyright{acmcopyright}
\begin{document}

\copyrightyear{2024}
\acmYear{2024}
\setcopyright{acmlicensed}\acmConference[WSDM '24]{Proceedings of the 17th ACM International Conference on Web Search and Data Mining}{March 4--8, 2024}{Merida, Mexico}
\acmBooktitle{Proceedings of the 17th ACM International Conference on Web Search and Data Mining (WSDM '24), March 4--8, 2024, Merida, Mexico}
\acmPrice{15.00}
\acmDOI{10.1145/3616855.3635840}
\acmISBN{979-8-4007-0371-3/24/03}

\title{Not All Negatives Are Worth Attending to: Meta-Bootstrapping Negative Sampling Framework for Link Prediction}

\author{Yakun Wang}
\authornote{Equal contributions.}
\affiliation{%
  \institution{Ant Group}
  \city{Beijing}
  \country{China}
  }
\email{feika.wyk@antgroup.com}

\author{Binbin Hu}
\authornotemark[1]
\affiliation{%
  \institution{Ant Group}
  \city{Hangzhou}
  \country{China}
  }
\email{bin.hbb@antfin.com}

\author{Shuo Yang}
\affiliation{%
  \institution{Ant Group}
  \city{Beijing}
  \country{China}
  }
\email{kexi.ys@antfin.com}

\author{Meiqi Zhu}
\affiliation{%
  \institution{Ant Group}
  \city{Beijing}
  \country{China}
  }
\email{zhumeiqi.zmq@antgroup.com}

\author{Zhiqiang Zhang}
\affiliation{%
  \institution{Ant Group}
  \city{Hangzhou}
  \country{China}
  }
\email{lingyao.zzq@antfin.com}

\author{Qiyang Zhang}
\affiliation{%
  \institution{Ant Group}
  \city{Shanghai}
  \country{China}
  }
\email{buxie.zqy@antgroup.com}

\author{Jun Zhou}
\affiliation{%
  \institution{Ant Group}
  \city{Beijing}
  \country{China}
  }
\email{jun.zhoujun@antfin.com}

\author{Guo Ye}
\affiliation{%
  \institution{Ant Group}
  \city{Hangzhou}
  \country{China}
  }
\email{yeguo.yg@antgroup.com}

\author{Huimei He}
\affiliation{%
  \institution{Ant Group}
  \city{Hangzhou}
  \country{China}
  }
\email{huimei.hhm@antgroup.com}

\renewcommand{\shortauthors}{Trovato and Tobin, et al.}

\begin{abstract}
The rapid development of graph neural networks (GNNs) encourages the rising of link prediction, achieving  promising performance with various applications.
Unfortunately, through a comprehensive analysis, we surprisingly find that current link predictors with dynamic negative samplers (DNSs) suffer from the migration phenomenon between ``easy'' and ``hard'' samples, which goes against the preference of DNS of choosing ``hard'' negatives, thus severely hindering capability.
Towards this end, we propose the MeBNS framework, serving as a general plugin that can potentially improve current negative sampling based link predictors.
In particular, we elaborately devise a Meta-learning Supported Teacher-student GNN (MST-GNN) that is not only built upon  teacher-student architecture for alleviating the migration between ``easy'' and ``hard'' samples but also equipped with a meta learning based sample re-weighting module for helping  the student GNN distinguish ``hard'' samples in a fine-grained manner. 
To effectively guide the learning of MST-GNN, we prepare a Structure enhanced Training Data Generator (STD-Generator) and an Uncertainty based Meta Data Collector (UMD-Collector) for supporting the teacher and student GNN, respectively. 
Extensive experiments show that the MeBNS achieves remarkable performance across six link prediction benchmark datasets.
\end{abstract}





\maketitle

\input{1_introduction}

\input{3_preliminary}

\input{4_MeBNS}

\input{5_experiment}
\input{2_relate_work}

\section{Conclusion}
In this paper, we propose the MeBNS framework, serving as a general plugin that can potentially improve current negative samplers based link predictors by alleviating the migration phenomenon and encouraging distinguishing ``hard'' samples in a fine-grained manner. 
 Extensive experiments demonstrate the effectiveness of the MeBNS framework.
 Our work imparts interesting findings about the migration phenomenon for DNS based link prediction and paves the way for future researchers to further expand upon this area, e.g., more subtle designs to close in the ideal sampling strategy $\mathcal{P}$  and more complicated graphs with multiple relations.


\clearpage
\bibliographystyle{ACM-Reference-Format}
\bibliography{reference}

\clearpage
\input{6_appendix}

\end{document}

%% file: 1_introduction.tex
\section{Introduction}
The prevalence of graph structures attracts a surge of investigation on graph data, enabling several downstream applications including personalized recommendation~\cite{ying2018graph,he2020lightgcn,hu2018leveraging,zang2023commonsense}, drug molecular design~\cite{drug_intro}, and supply chain finance~\cite{yang2021financial}. 
In general, these tasks could unfold as a typical instance of link prediction on graph data, which essentially depends on the effective characterization of interactive structure for target node pairs (i.e., edges). In recent years, graph neural networks (GNNs) \cite{gcn,geniepath,graphsage} have emerged as a promising direction for driving the rapid development of link prediction, which flexibly incorporates structural information into representation learning, and thus achieves dramatic performance~\cite{seal,pagnn,bellman-ford}.

In the typical flow of training a link predictor, only positive samples (i.e., observed edges) are involved, and considering all unobserved edges as negative samples is impractical for model learning in a real-world setting.   
Fundamentally, the quality of negative samples plays a decisive role in the capability of link predictors, which has recently spurred a fruitful line of related research \cite{yang2020understanding,chen2018fastgcn}. 
Earlier works commonly focus on static negative sampling with heuristic fixed distributions, including uniform~\cite{uniform1,uniform2} and degree-based distributions~\cite{pns,rendle2014improving}.  However, static negative samplers are incapable of adjusting the distributions of negatives within the training process, failing to explore more favorable negatives.
As a comparison, dynamic negative sampling (DNS) is receiving a growing amount of attention, which addresses the issue of dynamic adaptation. In particular, DNS helps link predictors prioritize so-called ``hard'' negatives based on predictions of the current model~\cite{dns3,noh2018improving},  structural correlations~\cite{recns,yang2020understanding,kgpolicy} and PageRank scores~\cite{pinsage,chen2019samwalker,wang2021samwalker++}. 

Despite the impressive improvement of DNS and its variants towards various applications, our empirical analysis on real-world datasets tells a rather different story. Specifically, we zoom into the capability of DNS in link prediction by conducting experiments on the Cora and CiteSeer datasets. 
{Surprisingly, our empirical analysis discloses that the migration phenomenon between ``easy'' and ``hard'' samples (defined in Sec. \ref{analysis})  goes against the core idea of DNS that only selects ``hard'' negatives in each iteration.} 
Intuitively, the issue may severely threaten the capability of the DNS based  link predictor, on account of abundant worthless optimizations during iteration. Therefore, we naturally concentrate on such an essential question:
\emph{Can we devise a bootstrapping~\footnote{In this paper, the term of \textit{bootstrap} is used in its idiomatic sense instead of the statistical sense.} framework tailored for improving the capability of negative sampler in link predictors by alleviating the migration phenomenon while retaining its original merits?}

In light of this discovery, the bottleneck of designing such a desired bootstrapping framework mainly lies in 
i)  alleviating the migration phenomenon
 and ii) encouraging link predictors to distinguish ``hard'' samples in a fine-grained manner. 
Towards this end, we propose a \textbf{Me}ta-\textbf{B}ootstrapping \textbf{N}egative \textbf{S}ampling framework (\textbf{MeBNS}), which could be plugged in current negative samplers with different link predictors for promising prospects. 
In particular, the Meta learning Supported Teacher-student GNN (\textbf{MST-GNN}) serves as the heart of the MeBNS framework, which is equipped with two designs to correspondingly address the above challenges: 
i) teacher-student architecture\footnote{We just reuse the ``teacher-student'' term in studies of knowledge distillation~\cite{hinton2015distilling,yim2017gift} for a clearer explanation, merely meaning the teacher GNN  distills ``hard'' samples to the student GNN in our paper.}. For alleviating the migration between ``easy'' and ``hard'' samples, it employs a well-trained teacher GNN to filter the ``easy'' samples and subsequently enforces a student GNN to specialize ``hard" samples.
ii) meta learning based sample re-weighting, which adopts a meta learner to impose samples with learnable meta weights, hopefully helping the student GNN distinguish ``hard'' samples in a fine-grained manner.
Moreover, the training pipeline of the teacher GNN is supported by the Structure enhanced Training Data Generator (\textbf{STD-Generator}), which flexibly injects structural information into current negative sampler.
Meanwhile, we  come up with an Uncertainty based Meta Data Collector (\textbf{UMD-Collector}) for guiding the learning of meta learner in the student GNN by collecting meta data with high confidence based on Dropedge based graph augmentation. 
To summarize, we make the following contributions. 
\begin{itemize}[leftmargin=*]
    \item \textbf{Inspirational Insights:} We highlight the critical demands of alleviating the migration phenomenon in current negative samplers for link prediction through empirical analysis. The issue severely hinders the capability of link predictors, which is unexplored in previous works.

    \item \textbf{General Framework:} We propose  a novel meta-bootstrapping negative sampling framework (MeBNS) for alleviates the intractable migration phenomenon.
    It is noteworthy that MeBNS not only theoretically guarantees a better optimization landscape, but also potentially serves as a general plugin, benefiting  existing negative samplers with various link predictors.
    \item \textbf{Multifaceted Experiments:} We conduct extensive experiments on six benchmark datasets, demonstrating that MeBNS surpasses several state-of-the-art negative samplers with  three classical link predictors as backbones. In-depth analysis reveals the contributions of MeBNS towards alleviating the migration phenomenon and specializing in  ``hard'' samples.
\end{itemize}

%% file: 3_preliminary.tex
\section{Preliminary and Analysis}
In this section, we first summarize the prevalent paradigm of a link predictor. Then we give an empirical analysis of the migration issue in DNS~\cite{dns1}, which potentially motivates the design of the Meta-Bootstrapping Negative Sampling (\textbf{MeBNS}) framework.

\subsection{Recap Link Predictor}
Let $\mathcal{G} = \left(\mathcal{V},\mathcal{E} \right)$ be an undirected graph consisting of the node set $\mathcal{V}$ and the edge set $\mathcal{E}$. In practice, each node  $v$ is commonly associated with a feature vector $\bm{x}_v \in \mathbb{R}^f$, where $f$ is the number of raw features. The goal of link prediction is learning a prediction function $\mathcal{F}(u,v|\mathcal{G}; \theta)$ to estimate the likelihood of whether there is a connection between the target node pair $(u, v)$, which commonly involves an encoder and a decoder component. 

\paratitle{Encoder}
With the recent advent of graph neural networks (GNNs), a predominant encoder of link prediction is the use of node-centric method~\cite{GCNbased,graphsage}, which applies the recursive neighborhood aggregation and representation updating on the whole graph (i.e. $\mathcal{G}$).
Specifically, we describe the detailed iterations of a GNN for a node $v$ as follows:
\begin{equation}
    \begin{split}
        \bm{h}^{(l+1)}_{v} & = \textit{Encoder}(\bm{h}_v^{(l)}|\mathcal{G}; \theta^{\mathcal{G}}) \\
                           & = \textit{UPD}(\bm{h}_v^{(l)}, \textit{AGG}(\bm{h}_{v'}^{(l)}|{v' \in \mathcal{N}_v}; \theta^{\mathcal{G}, \mathcal{A}}); \theta^{\mathcal{G}, \mathcal{U}}), \\
    \end{split}
\end{equation}
where $\bm{h}_v^{(l)}, \bm{h}_v^{(l+1)}$ respectively denote the node representations as the $l$ and $l+1$-th layer, and $\bm{h}_v^{(0)}$ is initialized with its original feature vector $\bm{x}_v$, and $\mathcal{N}_{v}$ represent the neighbors of node $v$. Here, $\textit{Encoder}(\cdot)$ is formally parameterized by $\theta^{\mathcal{G}}$, which, in particular, is comprised of the $\theta^{\mathcal{G}, \mathcal{A}}$-parameterized aggregation function $\textit{AGG}(\cdot)$ and $\theta^{\mathcal{G}, \mathcal{U}}$-parameterized updating function  $\textit{UPD}(\cdot)$.
As the core of GNNs, several efforts have been made towards the design of the aggregation and updating function~\cite{gat,geniepath}.

\paratitle{Decoder}
After $L$-layers iterations, a predictive decoder is built upon the final representations (i.e. $\bm{h}^{(L)}_{v}$) for each node pair $(u, v)$ as follows: 
\begin{equation}
    \textit{Decoder}(u,v|\mathbf{H}; \theta^{\mathcal{P}}) = \sigma(\bm{h}^{(L)}_{u}, \bm{h}^{(L)}_{v}; \theta^{\mathcal{P}}),
\end{equation}
where $\mathbf{H} \in \mathbb{R}^{|\mathcal{V}| \times d}$ is the collection of graph’s learned node representations with $d$-dimension in a matrix form.  $\sigma(\cdot)$ is implemented as the inner product in our paper, which supports fast retrieval.



Following common strategies in previous works \cite{kumar2020link}, a good link predictor could be effectively trained with cross entropy function with negative sampling in a point-wise manner:
\begin{equation}
    \min_{\theta} =\frac{1}{|\mathcal{O}|} \sum_{(u, v, y_{u,v}) \in \mathcal{O}}{\mathcal{L}(\mathcal{F}(u, v|\mathcal{G}; \theta), y_{u, v})},
\end{equation}
where the training data $\mathcal{O} = \mathcal{O}^+ \cup   \mathcal{O}^-$  includes observed edges as positive samples $\mathcal{O}^+ = \{(u, v, y_{u,v})| (u, v) \in \mathcal{E}\}$ and unobserved edges as negative samples  $\mathcal{O}^- =  \{(u, v, y_{u,v})| (u, v) \in  \mathcal{V} \times \mathcal{V} \backslash \mathcal{E}\}$, where $y_{u,v} = 1$ in $\mathcal{O}^+$ and $y_{u,v} = 0$ in  $\mathcal{O}^-$. Since $|\mathcal{O}^-|$ is usually huge, making the model training impractical, the design of an effective negative sampler has served as the fundamental driving force behind the promising performance of link predictors.


On the other hand, edge-centric link predictor has recently made exceptional progress, including SEAL~\cite{seal}, GRAIL~\cite{grail}, PaGNN~\cite{pagnn} and NBFNet~\cite{bellman-ford}. Specifically,  this line of methods enables interactive structure learning by extracting enclosing subgraph $\mathcal{G}_{u, v}$ for each pair of node $(u, v)$, whose encoder could be formalized as  $\text{Encoder}(\bm{h}_v^{(l)} | \mathcal{G}_{u,v}; \theta^{\mathcal{G}})$. 
It is worthwhile mentioning that our MeBNS framework is model-agnostic, which can naturally benefit both node- and edge-centric link predictors. For brevity, we introduce the MeBNS framework with the backbone of node-centric link predictor, which could be easily analogized to edge-centric backbones.
\\

\begin{figure}[htbp]
        \centering
        \subfigure[AUC and Loss with DNS sampler ]{\includegraphics[width=1.67in]{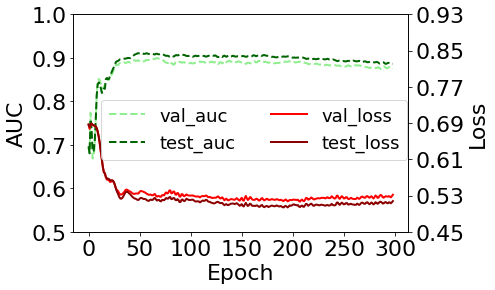}}  
        \subfigure[Impact of removing ``easy'' samples ]{\includegraphics[width=1.5in]{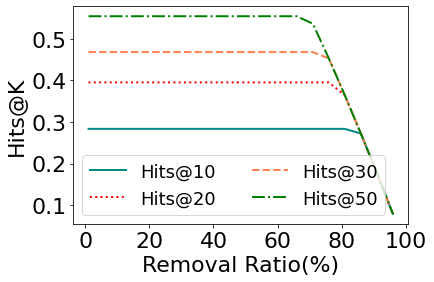}} 
        \subfigure[Migration between 100-th and 150-th Epoch]{\includegraphics[width=1.6in]{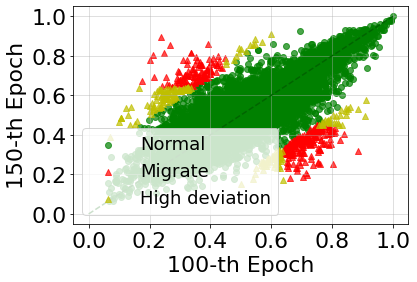}}
        \subfigure[Migration between 150-th and 250-th Epoch]{\includegraphics[width=1.6in]{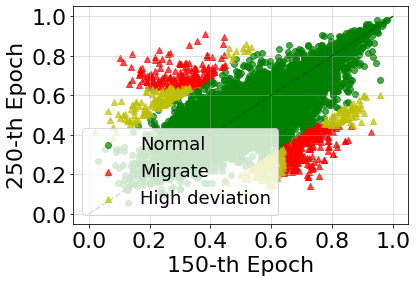}}
        \caption{Empirical analysis with DNS sampler.}
        \label{fig:motivation_anly}
\end{figure}

\subsection{Empirical Analysis on Negative Sampler}
\label{analysis}
Intuitively, the performance of link prediction greatly hinges on the quality of the negative sampler since only observed edges are involved in the training process.  
Compared to static samplers in earlier works \cite{uniform1,pns,rendle2014bayesian}, recent evidence have demonstrated the superior performance of the dynamical negative sampling (DNS) strategy, which dynamically
collects samples with high prediction score to form the set of negatives $\mathcal{O}^-$. These strategies are mainly based on predictions of the current models~\cite{dns1,wang2017irgan}, graph structures~\cite{yang2020understanding,recns}, GAN~\cite{wang2017irgan,cai2017kbgan,wang1711graphgan,hu2019adversarial} and mixing techniques~\cite{huang2021mixgcf,mixmethod}.
We direct the readers to Sec.~\ref{sec_ap_rel} for a comprehensive discussion of related work toward link prediction and negative sampling. 
Although DNS and its variants have been widely adopted in various applications due to its dramatic performance, our empirical analysis on real-world datasets tells a rather different story.

As a motivated example, we conduct a link prediction experiment with DNS on the Cora dataset, and corresponding experimental settings are detailed in the ``Experiment'' and ``Appendix'' parts.  Firstly, we present the learning curve of the link predictor in Fig 1 (a), {in terms of the change of the loss and performance in the validation and test set. From the plot, we find the model rapidly converges within about 100 epochs and then the loss  remains at a certain level.} To further explore whether the link predictor with DNS sampler has converged to the optimal condition,  
we record the prediction scores in several specific epochs (i.e. Epoch 100, 150, and 250) and illustrate their correlations in Fig. 1 (c) and (d). Concretely, a sample in the figure is marked in green when the gap of its prediction score between two epochs is less than a fixed threshold, and marked in yellow or red otherwise~\footnote{In the experiment, we set the threshold of the gap as 0.3.}. In the ideal situation, samples are expected to be distributed near the diagonal (i.e., all samples are marked as green dot (``\textcolor{green}{\CIRCLE}'') in Fig. 1 (c) and (d)) if the model has converged. Unfortunately, things go contrary to our wishes since there exist a number of samples with high deviation, marked as yellow (``\textcolor{yellow}{\UParrow}'') or red (``\textcolor{red}{\UParrow}'') triangle.  What's worse, DNS based methods have an intractable migration phenomenon, which means that a sample may be an ``easy'' or ``hard'' negative in a certain epoch (e.g., Epoch 100) while being an ``hard'' or ``easy'' negative in the subsequent epoch (e.g., Epoch 150), and we mark these samples in red as to distinguish from yellow. Here, in order to understand the migration between ``hard'' and ``easy'' samples more intuitively which are marked in red, we specify the definition that ``hard'' samples are those mistakenly considered as false positives ~\footnote{We define negative samples greater than the threshold as false positives. The threshold here is the optimal value from ROC Curve.} and  ``easy'' samples  are the remaining 50\% of the tail true negatives with lower prediction scores and farther from the decision boundary.
In summary, the above empirical analysis provides a surprising observation and the corresponding inspiration.

\paratitle{Observation}
Surprisingly, the migration phenomenon between ``easy'' and ``hard'' goes against the core idea of DNS (i.e. only select ``hard'' negatives in each iteration). It makes the DNS based link predictor mainly focuses on the worthless optimization after a few epochs, resulting in sub-optimal performance. 

\paratitle{Inspiration}
{A natural idea to tackle migration issues is enforcing the link predictor to optimize the ``hard'' samples after a certain converged state since ``easy'' samples may contribute little in the later stages of optimization. Fortunately, the observation in Fig. 1(b) also gives the encouraging result that the performance would be completely unaffected when 65\% of ``easy'' samples are removed.}

%% file: 4_MeBNS.tex
\section{The Proposed Framework}
In this section, we dive into the elaboration of the proposed novel MeBNS framework, as presented in Fig.~\ref{fig:MeBNS_framework}.

\begin{figure}[t]
    \centering
    \includegraphics[width=3.4in]{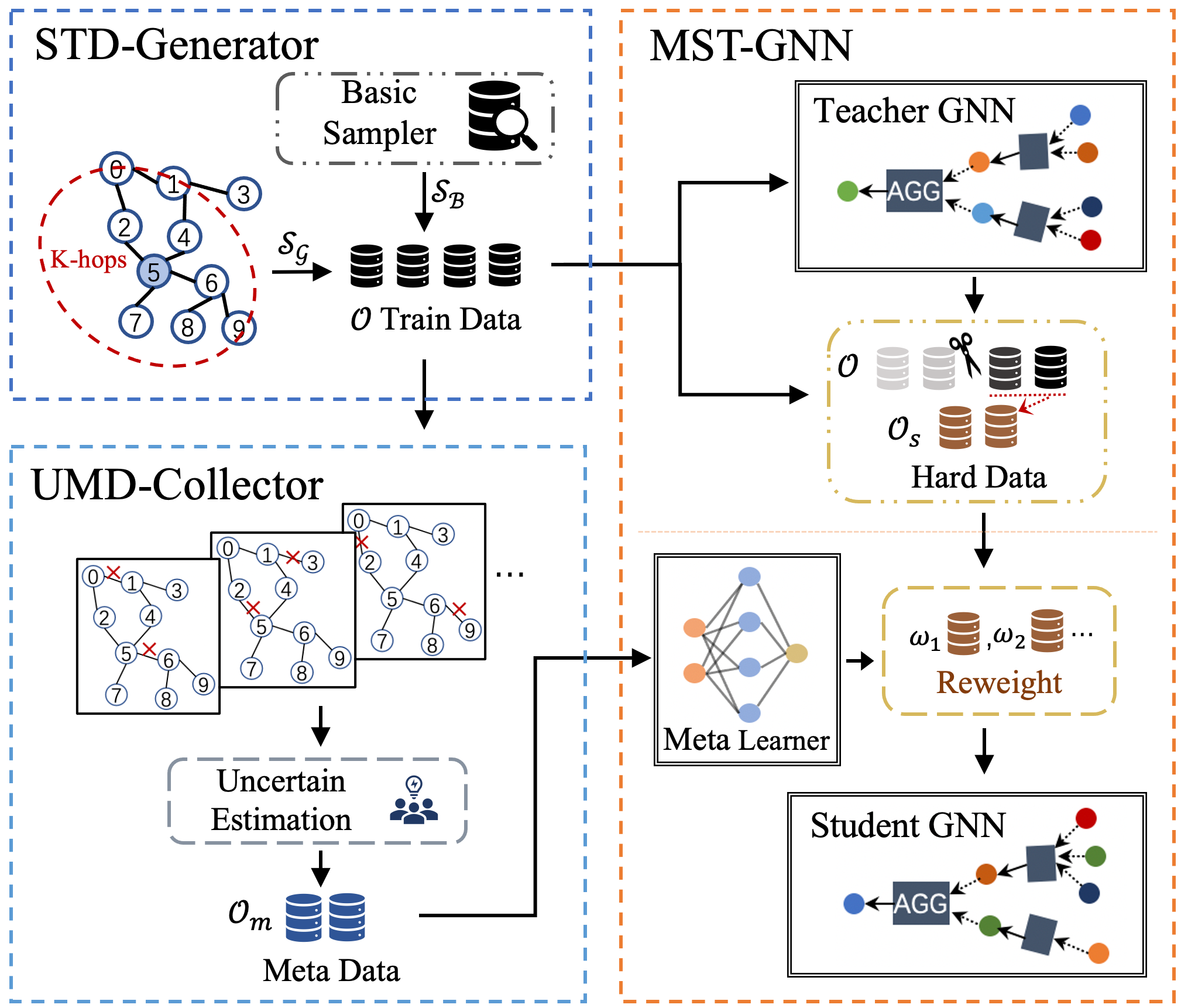}
    \caption{Overview of the proposed MeBNS model.}
    \label{fig:MeBNS_framework}
\end{figure}

\subsection{Overall Pipeline of Meta learning Supported Teacher-student GNN (MST-GNN)}
With the encouraging inspiration derived from our empirical analysis of current negative samplers, the core of the MeBNS framework is the elaborate MST-GNN architecture, which employs a teacher GNN to filter ``easy'' negatives, while a student GNN, equipped with a well-designed meta learner, is applied to further improve the capability on ``hard'' negative with learnable meta weights.

In particular, our MST-GNN involves two GNNs with the same architecture, detailed in the ``Preliminary'' part. Since our MeBNS framework is model-agnostic and can be directly plugged into current link predictors, we would not present its specific instantiation in this paper and just denote the teacher and student GNN as $\mathcal{F}^{t}(u,v|\mathcal{G}; \theta_{t})$ and $\mathcal{F}^{s}(u,v|\mathcal{G}; \theta_{s})$, respectively.
As mentioned above, the teacher GNN is responsible for converging to a stable state and then filtering a part of ``easy'' samples. Hence, the optimal teacher GNN can be easily obtained by minimizing the following loss:
\begin{equation}
\theta_{t}^{*} = \mathop{argmin} \limits_{\theta_{t}} \frac{1}{|\mathcal{O}|} \sum_{(u, v, y_{u,v}) \in \mathcal{O}} \mathcal{L}(\mathcal{F}^{t}(u,v|\mathcal{G}; \theta_{t}),y_{u,v}) 
\label{func:teacher_loss}.
\end{equation}
Here, the capability of the teacher GNN greatly hinges on the construction of training data (i.e., $\mathcal{O}$), which is derived from STD-Generator.

\subsection{Structure Enhanced Training Data Generator (STD-Generator)}
In the training data (i.e., $\mathcal{O} = \mathcal{O}^+ \cup \mathcal{O}^-$), following the common way in previous works, we regard the observed edges in graphs as positive samples (i.e.,  $\mathcal{O}^+$). In terms of the negatives (i.e., $\mathcal{O}^-$), we directly inherit the merits from the current negative sampler and adopt them as the base negative sampler, denoted as $\mathcal{S}_\mathcal{B}(u)$ with the center node $u$. On the other hand, with the usefulness of the structure prior in graphs, STD-Generator comes up with injecting more informative negatives with the help of the graph structure. Specifically, we denote the graph based negative sampler as $\mathcal{S}_\mathcal{G}(u)$ for the center node ${u}$, which aims at collecting $K$-hops neighbors for node $u$. Hence, for a center node $u$, its negatives $v'$ could be obtained by combining the base and graph negative sampler with a probability of $\delta$ as follows:
\begin{equation}
    \begin{cases}
v' \sim  \mathcal{S}_\mathcal{G}(u), & \text{with the probability of $\delta$}; \\
v' \sim \mathcal{S}_\mathcal{B}(u), & \text{with the probability of $1 - \delta$},
\end{cases}
\end{equation}
where $\delta$ is a hyper-parameter, controlling the proportion of negatives from the graph based negative sampler.

\subsubsection{Filtering ``easy'' Samples for the Student GNN}
The aforementioned observations have suggested that removing a part of ``easy'' samples would alleviate the unsatisfying migration phenomenon in current negative samplers for link prediction without harming the performance, and hopefully benefit the student GNN for focusing on ``hard'' samples.
Specifically, given a well-trained teacher model $\mathcal{F}^{t}(u,v|\mathcal{G}; \theta^{T}_{t})$ after $T$-epochs iteration, we firstly perform inference on train data $\mathcal{O}$, denoting the prediction result as $\mathcal{R}$,  and then  construct the sample set for the student GNN as follows:
\begin{equation}
    \mathcal{O}_s = \{(u, v, y_{u,v}) \in \mathcal{O} | \mathcal{R}_{u, v} \geq \zeta(\mathcal{R}, \beta)\},
\label{func:data_split}
\end{equation}
where $\zeta(\mathcal{R}, \beta)$ finds value of the $\lfloor  \beta |\mathcal{O}| \rfloor$-largest prediction score from  $\mathcal{R}$, where $\beta$ is a hyper-parameter controlling the ``hardness'' of samples transferred to the student GNN.

As mentioned above, since the student GNN utilizes the same structure as the teacher GNN, and thus we naturally initialize its parameters $\theta_{s}$ with  $\theta^{T}_{t}$. In order to help the student GNN well distinguish ``hard'' negative, we introduce a meta learner $\mathcal{Z}(u,v|\mathcal{G}; \Delta)$ to emphasize important samples in an adaptive manner. Hence, the optimization of student GNN can be formalized as follows:
\begin{equation}
    \theta_{s}^{*}(\Delta) 
= \mathop{argmin} \limits_{\theta_{s}} \frac{1}{|\mathcal{O}_s|} \sum_{(u, v, y_{u,v}) \in \mathcal{O}_s} \omega_{u,v}*\mathcal{L}(\mathcal{F}^{s}(u,v|\mathcal{G}; \theta_{s}),y_{u,v}) ,
\label{student_orginal}
\end{equation}
where $\omega_{u,v}$ is automatically learned via the meta learner  $\mathcal{Z}(u,v|\mathcal{G}; \Delta)$ to emphasize and balance ``hard'' samples in the link prediction for the student GNN.
 
\subsection{Diving into the Meta Learner for Sample Re-weighting in the Student GNN}
Hopefully, the meta learner is designed to help the student GNN to locate a more crisp decision boundary. Actually, meta learning \cite{shu2019meta,ren2018learning} has become a fast-growing research area in recent years, which is demonstrated to have rapid adaptability to new tasks only with few data and support robust training on biased data.
Here, we borrow the main idea to automatically learn a $\Delta$-parameterized re-weighting function $\mathcal{Z}(u,v|\mathcal{G}; \Delta)$ tailored for link prediction in a meta-learning manner.
Specifically, due to its superior capability of complicated interaction, we implement the meta learner \cite{shu2019meta} as two layers MLPs with a ReLU activation function and an output layer with the sigmoid function.

\subsubsection{Uncertainty based Meta Data Collector (UMD-Collector)}
Generally, meta data is crucially important to guide the learning of the meta learner, and samples with high confidence are prone to benefit the stability of the learning procedure. Inspired by the observations in Fig.~\ref{fig:motivation_anly} (c), (d), we find samples with high confidence are always corresponding to the low deviation between two epochs, i.e., low uncertainty. Hence, we shift attention towards  DropEdge based graph augmentation, which could be interpreted as an estimation of uncertainty \cite{liu2022confidence}. Formally, we describe the augmentation function $\Gamma(\cdot)$ as follows:
\begin{equation}
    \Gamma(\mathcal{G}) = \{\mathcal{V}, \mathcal{E} \odot \bm{Q}(\rho)\} \ \ \ \ \rho \sim \text{Gaussian}(0, 1),
\end{equation}
where $\bm{Q}(\rho) \in \{0, 1\}^{|\mathcal{E}|}$ is a mask vector, i.e. each edge is
discarded from the graph with the probability of $\rho$ from a Gaussian distribution. 
Following the core idea of Monte-Carlo sampling, the above augmentation is executed $N$ times and we perform link prediction under each augmented graph $\mathcal{G}' \sim \Gamma(\mathcal{G})$ as 
$p_{u,v}(\mathcal{G'}) = \mathcal{F}^{t}(u,v|\mathcal{G}'; \theta_{t})$. Subsequently,
the uncertainty for a node pair $(u, v)$ is estimated as follows:
\begin{equation}
    \begin{split}
         \mathcal{B}_{u, v} &= \mathbb{E}_{\mathcal{G}' \sim \Gamma(\mathcal{G})}(p_{u,v}(\mathcal{G'}) -  \mathbb{E}_{\mathcal{G}' \sim \Gamma(\mathcal{G})}(p_{u,v}(\mathcal{G'}))^2 \\
         &\approx \frac{1}{N}\sum_{\mathcal{G}' \sim \Gamma(\mathcal{G})}p^2_{u,v}(\mathcal{G'}) -  \frac{1}{N^2}{(\sum_{\mathcal{G}' \sim \Gamma(\mathcal{G})}p_{u,v}(\mathcal{G'}))^2}\\
    \end{split}
\end{equation}
Therefore, the meta data is constructed as $\mathcal{O}_m = \{(u, v, y_{u, v}) \in \mathcal{O} | \mathcal{B}_{u,v} < \tau\}$, where $\tau$ is a fixed threshold.

\subsubsection{Towards the Optimization of Student GNN with Meta Data}

Since the above student optimization in Eq.~\ref{student_orginal} consists of meta learner parameters, the optimization of student GNN can be decomposed into two nested optimizations. The inner optimization aims to optimize the meta learner based on a single-step update of the student and the outer optimization performs optimizations on the student with the latest meta parameters.

\paratitle{Inner Optimization}
Following the normal meta manner, we first explicitly calculate a single-step gradient on ``hard'' data for further performing second-order gradient optimization on meta data. The explicit calculation is as follows:
\begin{equation}
\small
\theta_{s}^{'}(\Delta)=\theta_{s}^{k} - \nabla \frac{1}{|\mathcal{O}_s|} \sum_{(u, v, y_{u,v}) \in \mathcal{O}_s} \omega_{u,v}^{k}* \mathcal{L}(\mathcal{F}^{s}(u,v|\mathcal{G}; \theta_{s}^{k}),y_{u,v}),\\
\label{func:inner_opt}
\end{equation}
where $\theta_{s}^{k}$, $\omega_{u,v}^{k}$ denote the optimal student parameters and sample weights respectively at step k.

Then the meta learner optimize by using meta data $\mathcal{O}_m$ with the following equation as $\Delta$ is contained in $\theta_{s}^{'}(\Delta)$:
\begin{equation}
\small
\Delta^{k+1}=\Delta^{k} - \nabla \frac{1}{|\mathcal{O}_m|} \sum_{(u, v, y_{u,v}) \in \mathcal{O}_m} \mathcal{L}(\mathcal{F}^{s}(u,v|\mathcal{G}; \theta_{s}^{'}(\Delta),y_{u,v}).
\end{equation}

\paratitle{Outer Optimization}
The outer optimization of the student can be calculated through the weighted loss defined in Eq.\ref{student_orginal} with the latest sample weights $\omega_{u,v}^{k+1} = \mathcal{Z}(u,v|\mathcal{G}; \Delta^{k+1})$.
\begin{equation}
\small
\theta_{s}^{k+1}=\theta_{s}^{k} - \nabla \frac{1}{|\mathcal{O}_s|} \sum_{(u, v, y_{u,v}) \in \mathcal{O}_s} \omega_{u,v}^{k+1} * \mathcal{L}(\mathcal{F}^{s}(u,v|\mathcal{G}; \theta_{s}^{k}),y_{u,v}).\\
\label{func:meta_opt}
\end{equation}


\begin{algorithm}[t]
    \caption{Training Pipeline of MeBNS Framework}
    \label{alg:algorithm}
    \begin{algorithmic}[1] 
    \REQUIRE{Graph $\mathcal{G}=(\mathcal{V},
    \mathcal{E})$, Positive Data $\mathcal{O}^{+}$}, Teacher GNN: $\mathcal{F}^{t}(u,v|\mathcal{G}; \theta_{t})$, Student GNN: $\mathcal{F}^{s}(u,v|\mathcal{G}; \theta_{s})$, Meta Learner: $\mathcal{Z}(u,v|\mathcal{G}; \Delta)$\\
    \ENSURE{ Well-trained student GNN: $\mathcal{F}^{s}(u,v|\mathcal{G};\theta_{s})$}
        \STATE \textcolor{red}{\textbf{/** Teacher GNN Training **/}}
        \FOR {epoch $<$ $T$ }
        \STATE $\mathcal{O}^-  \gets \textbf{STD-Generator} $ 
        \STATE \textit{Training data} $\mathcal{O} = \{\mathcal{O}^+ \cup \mathcal{O}^-\} $ 
        \STATE Update Teacher GNN $\theta_{t}$ on $\mathcal{O}$ with Eq.~\ref{func:teacher_loss}
        \ENDFOR
        \STATE \textcolor{blue}{/** \textbf{Student GNN Training} **/}
        \STATE Initialize Student GNN $\theta_{s} \gets \theta_{t}^{T}$  
        \STATE $\mathcal{O}_m  \gets \textbf{UMD-Collector} $
        \WHILE {not converge}
        \STATE $\mathcal{O}^-  \gets \textbf{STD-Generator} $
        \STATE \textit{Training data} $\mathcal{O} = \{\mathcal{O}^+ \cup \mathcal{O}^-\} $
        \STATE Obtain $\mathcal{O}_s \sim \mathcal{O}$ with Eq.~\ref{func:data_split}
        
        \STATE Update $\theta_{s}$ and $\Delta$ on \{$\mathcal{O}_s$, $\mathcal{O}_m$\} with Eq.~\ref{func:inner_opt}-Eq.~\ref{func:meta_opt}.
        \ENDWHILE
    \end{algorithmic}
\end{algorithm}

\subsection{Why It Works}
\subsubsection{Theoretical Justification}

For the sake of convenience, in our theoretical analysis, we choose to maximize the average utility~\cite{hacohen2019power} $\mathcal{U}_\theta({\mathcal{G}}) = e^{-\mathcal{L}_\theta(\mathcal{G})}$ of the observed graph for choosing the optimal hypothesis (i.e., GNN) for link prediction. 
For the teacher GNN $\mathcal{F}^{t}(u,v|\mathcal{G}; \theta)$, the optimal $\theta^*$ from the data could be computed by maximizing the following formulation:
\begin{equation}
    \mathcal{R}^t(\theta) = \mathbb{E}^t(\mathcal{U}_\theta({\mathcal{G}})) = \sum_{(u, v) \in \mathcal{O}}\mathcal{U}_\theta(u, v) \triangleq \sum_{(u, v) \in \mathcal{O}}e^{-\mathcal{L}_\theta(u, v)} \\
\end{equation}

In terms of the student GNN, we regard it as imposing a prior distribution on the observed data $\mathcal{O}$ for data sampling (e.g., DNS). Hence, this can be formalized as follows: 
\begin{equation}
    \mathcal{R}^s(\theta) = \mathbb{E}^s(\mathcal{U}_\theta({\mathcal{G}})) = \sum_{(u, v) \in \mathcal{O}}p_{u, v} \cdot \mathcal{U}_\theta(u, v) \triangleq \sum_{(u, v) \in \mathcal{O}}p_{u, v} \cdot e^{-\mathcal{L}_\theta(u, v)} \\
    \label{equ_stu}
\end{equation}

Then, we have the following theorem to depict the merit of the proposed MeBNS, potentially contributing to the learning.
\begin{theorem}
With the ideal sampling strategy $\mathcal{P}$ on observed data space $\mathcal{O}$ satisfying  (i) the non-increasing property of the difficulty level of $(u, v)$ pair  and 
(ii) the positive correlation with the optimal utility $\mathcal{U}_{\theta^*}$,
compared to  $\mathcal{F}^{s}(u,v|\mathcal{G}; \theta)$, the optimization landscape of $\mathcal{F}^{t}(u,v|\mathcal{G}; \theta)$ is inclined to enlarge the difference between the optimal parameters
and all other parameter values whose covariance with the optimal solution is smaller than the variance of the optimum. More specifically,
\begin{equation}
    \mathcal{R}^s(\theta^*) - \mathcal{R}^s(\theta) \geq   \mathcal{R}^t(\theta^*) - \mathcal{R}^t(\theta), \ \ \ \  \forall \theta:\ \  \text{Cov}(\mathcal{U}_{\theta^*}, \mathcal{U}_{\theta}) \leq \text{Var}(\mathcal{U}_{\theta^*}) \\
\end{equation}
\end{theorem}

\begin{proof}
Firstly, we rewrite the Eq.~\ref{equ_stu} as follows:
\begin{equation}
    \begin{split}
        \mathcal{R}^s(\theta) &=  \sum_{(u, v) \in \mathcal{O}}p_{u, v} \cdot \mathcal{U}_\theta(u, v) \\
        & = \sum_{(u, v) \in \mathcal{O}}(p_{u, v} - \mathbb{E}(\mathcal{P})) \cdot (\mathcal{U}_\theta(u, v) - \mathbb{E}(\mathcal{U}_{\theta}))  + |\mathcal{O}|\mathbb{E}(\mathcal{P}) \mathbb{E}(\mathcal{U}_{\theta}) \\
        & = \text{Cov}(\mathcal{U}_{\theta}, \mathcal{P}) + |\mathcal{O}|\mathbb{E}(\mathcal{P}) \mathbb{E}(\mathcal{U}_{\theta}) \\
        & \triangleq \mathcal{R}^t(\theta) + \text{Cov}(\mathcal{U}_{\theta}, \mathcal{P}), \\
        \label{eq_proof_1}
    \end{split}
\end{equation}
where $\text{Cov}(\mathcal{U}_{\theta}, \mathcal{P})$ represents the the covariance between the two random variables $\mathcal{U}_{\theta}$ and $\mathcal{P}$ on graph data space $\mathcal{O}$. 

Generally, the non-increasing sampling strategy $\mathcal{P}$ of the difficulty level of $(u, v)$ pair on observed data space $\mathcal{O}$ implies the positive correlation between  $\mathcal{P}$ and the utility $\mathcal{U}_{\theta}$. And in an ideal condition,  $\mathcal{P}$ is expected to be positively correlated to the utility with optimal parameters, i.e., $\mathcal{U}_{\theta*}$
In particular, we define the sampling strategy $\mathcal{P}$ in the ideal case as follows:

\begin{equation}
    p_{u, v} \propto \exp^{-\mathcal{L}_{\theta^*}(u,v)} / \sum_{(u', v') \in \mathcal{O}}{\exp^{-\mathcal{L}_{\theta^*}(u',v')}}.
\end{equation}
Then, we rewrite Eq.~\ref{eq_proof_1} as follows:
\begin{equation}
    \mathcal{R}^s(\theta) \triangleq \mathcal{R}^t(\theta) + \kappa \cdot \text{Cov}(\mathcal{U}_{\theta}, \mathcal{U}_{\theta^*}),
\end{equation}
where $\kappa = 1 / \sum_{(u', v') \in \mathcal{O}}{\exp^{-\mathcal{L}_{\theta^*}(u',v')}}$.

Naturally, $\mathcal{R}^s(\theta)$  with the optimal parameter $\theta^*$ can be denoted as follows:
\begin{equation}
    \begin{split}
        \mathcal{R}^s(\theta^*) & = \mathcal{R}^t(\theta^*) + \kappa \cdot \text{Cov}(\mathcal{U}_{\theta^*}, \mathcal{U}_{\theta^*}) \\
         & = \mathcal{R}^t(\theta^*) + \kappa \cdot \text{Var}(\mathcal{U}_{\theta^*}). \\
    \end{split}
    \label{eq_proof6}
\end{equation}
On the other hand, 
\begin{equation}
    \begin{split}
        \mathcal{R}^s(\theta) &=  \mathcal{R}^t(\theta) +  \kappa \cdot  \text{Cov}(\mathcal{U}_{\theta}, \mathcal{U}_{\theta^*}). \\
    \end{split}
     \label{eq_proof7}
\end{equation}
By combining Eq.~\ref{eq_proof6} and Eq.~\ref{eq_proof7}, we conclude:
\begin{equation}
    \begin{split}
         \mathcal{R}^s(\theta^*) - \mathcal{R}^s(\theta) &= \mathcal{R}^t(\theta^*) -  \mathcal{R}^t(\theta) + \kappa \cdot (\text{Var}(\mathcal{U}_{\theta^*}) -  \text{Cov}(\mathcal{U}_{\theta}, \mathcal{U}_{\theta^*})) \\
         & \geq \mathcal{R}^t(\theta^*) - \mathcal{R}^t(\theta) \\
    \end{split}
\end{equation}
from the facts (i) $\forall \theta:\ \  \text{Cov}(\mathcal{U}_{\theta^*}, \mathcal{U}_{\theta}) \leq \text{Var}(\mathcal{U}_{\theta^*})$ and (ii) $\kappa = 1 / \sum_{(u', v') \in \mathcal{O}}{\exp^{-\mathcal{L}_{\theta^*}(u',v')}}$ is surely positive.
\end{proof}

In short, the theorem clearly tells that our MeBNS framework generally obtains a better landscape of a link predictor, specifically making the optimization overall steeper. In particular, MeBNS is equipped with the remarkable ability to refine the optimization when the optimal parameter and other parameters are totally uncorrelated, and even negatively correlated. 
Hence, compared to traditional link predictors, MeBNS could achieve a further convergence with alleviating the migration phenomenon (See empirical evidence in Fig. \ref{fig_qua} (a) \& (c)). It is worthwhile to note the training pipeline of MeBNS naturally not only follows the non-increasing sampling strategy( i.e. ``easy'' (teacher GNN) $\rightarrow$ ``hard'' (student GNN)), but also dedicates to imposing the sampling strategy close to the ideal case by over-sampling ``hard'' samples for student GNN with a converged teacher GNN.

\subsubsection{Complexity Analysis\label{sec_ca}}
The overall training pipeline of  MeBNS framework is presented in Algorithm.~\ref{alg:algorithm}
Since MeBNS divides the whole training process into teacher and student phases, we separately analyze the time complexity of the two phases.
Taking a single iteration for example, the training time complexity of teacher GNN $\Upsilon_t \left(\mathcal{V},\mathcal{E} \right)$ is actually the same as the baseline $\Upsilon_b \left(\mathcal{V},\mathcal{E} \right)$.
Extra computation occurs in the student phase due to the secondary gradient calculation on meta data. 
As shown in Eq.~\ref{func:inner_opt}-Eq.~\ref{func:meta_opt}, two back-propagation operations are required. Without loss of generality, time complexity of MeBNS is near $ \Upsilon \left(\mathcal{V},\mathcal{E} \right)$ = $\kappa \cdot \Upsilon_b \left(\mathcal{V},\mathcal{E} \right)$, where $\kappa$ is a constant ($\leq 4$ or $\approx 4$).
In summary, the complexity of MeBNS is linearly comparable with the traditional negative sampling based link prediction. Through the evaluation on Pubmed dataset, we could observe $\kappa = 0.12$ with GCN backbone, $\kappa = 0.43 $ with GAT backbone and $\kappa = 4.1$ with SEAL backbone, which reveals the scalability and efficiency of the proposed MeBNS.  And a large OGB dataset is adopted for evaluation to prove this.

%% file: 5_experiment.tex
\section{Experiments}


\begin{table}[t]
    \caption{The statistics of the datasets.}
    \centering
     \setlength{\tabcolsep}{3.5mm}{
    \begin{tabular}{cccc}
        \toprule
       Dataset & Nodes & Edges & Features   \\
       \midrule
         Cora & 2, 708 & 10, 556 & 1, 433  \\
         CiteSeer & 3, 327 & 9, 104 & 3, 703 \\
         Pubmed & 19, 717 & 88, 648 & 500 \\
         Amazon Photo & 7, 650 & 238, 162 & 745  \\
         Facebook & 4, 039 & 88, 234 & 1, 283  \\
         OGB-DDI & 4, 267 & 1, 334, 889 & N.A.  \\
    \bottomrule 
    \end{tabular}}
    \label{tab:dataset}
\end{table}

\begin{table*}[t]
    \caption{Performance comparison on six datasets with GCN based backbone. We use the \colorbox[gray]{0.8}{shadow} to mark that MeBNS benefits the original samplers and the best performance is \textbf{\color{blue}{bolded in blue}}.}
    \centering
    \setlength{\tabcolsep}{1.2mm}{
    \begin{tabular}{c|c|cc|cc|cc|cc|cc|cc}
    \toprule
        {Datasets} & {Metrics} & Uniform & +MeBNS & PNS & +MeBNS & DNS & +MeBNS & SANS & +MeBNS & RecNS & +MeBNS & MCNS & +MeBNS \\
        \midrule
        \multirow{3}{*}{Cora} & Hits@20 & 0.3450 & \cellcolor[gray]{0.8}{0.5815} & 0.4094 & \cellcolor[gray]{0.8}{0.5023} & 0.4113 & \cellcolor[gray]{0.8}{0.5184} & 0.2009 & \cellcolor[gray]{0.8}{0.4748} & 0.2056 & \cellcolor[gray]{0.8}{0.2398} & 0.5535 & \cellcolor[gray]{0.8}{\textbf{\color{blue}{0.6075}}}       \\
                    & Hits@30 
                    & 0.4436 & \cellcolor[gray]{0.8}{0.6151} & 0.4928 & \cellcolor[gray]{0.8}{0.5772} & 0.4815 & \cellcolor[gray]{0.8}{0.5734} & 0.2682 & \cellcolor[gray]{0.8}{0.5289} & 0.2606 & \cellcolor[gray]{0.8}{0.2862} & 0.6388 & \cellcolor[gray]{0.8}{\textbf{\color{blue}{0.6540}}}        \\
                    & AUC  
                    & 0.9005 & \cellcolor[gray]{0.8}{0.9117} & 0.9147 & \cellcolor[gray]{0.8}{0.9224} & 0.9031 & \cellcolor[gray]{0.8}{0.9114} & 0.8349 & \cellcolor[gray]{0.8}{0.8825} & 0.8106 & \cellcolor[gray]{0.8}{0.7940} & 0.9206 & \cellcolor[gray]{0.8}{\textbf{\color{blue}{0.9302}}}       \\
        \midrule
        \multirow{3}{*}{CiteSeer} & Hits@20 & 0.4758 & \cellcolor[gray]{0.8}{0.5549} & 0.4747 & \cellcolor[gray]{0.8}{0.5406} & 0.4857 & \cellcolor[gray]{0.8}{0.5450} & 0.1516 & \cellcolor[gray]{0.8}{0.4362} & 0.2164 & \cellcolor[gray]{0.8}{0.3230} & 0.5670 & \cellcolor[gray]{0.8}{\textbf{\color{blue}{0.6241}}}       \\
                    & Hits@30 & 0.5395 & \cellcolor[gray]{0.8}{0.5846} & 0.5439 & \cellcolor[gray]{0.8}{0.6098} & 0.5890 & \cellcolor[gray]{0.8}{0.6043} & 0.2373 & \cellcolor[gray]{0.8}{0.4879} & 0.2769 & \cellcolor[gray]{0.8}{0.4054} & 0.6219 & \cellcolor[gray]{0.8}{\textbf{\color{blue}{0.6934}}}         \\
                    & AUC    &  0.8936 & \cellcolor[gray]{0.8}{0.9138} & 0.9191 & \cellcolor[gray]{0.8}{0.9244} & 0.9047 & \cellcolor[gray]{0.8}{0.9120} & 0.7874 & \cellcolor[gray]{0.8}{0.8439} & 0.8324 & \cellcolor[gray]{0.8}{0.8636} & 0.9319 & \cellcolor[gray]{0.8}{\textbf{\color{blue}{0.9399}}}       \\
        \midrule
        \multirow{3}{*}{Pubmed} & Hits@20 & 0.3333 & \cellcolor[gray]{0.8}{\textbf{\color{blue}{0.4487}}} & 0.2814 & 0.2787 & 0.3272 & \cellcolor[gray]{0.8}{0.3381} & 0.1327 & \cellcolor[gray]{0.8}{0.2898} & 0.0822 & \cellcolor[gray]{0.8}{0.0856} & 0.2824 & \cellcolor[gray]{0.8}{0.3018}       \\
                    & Hits@30 &  0.3925 & \cellcolor[gray]{0.8}{\textbf{\color{blue}{0.5004}}} & 0.3410 & 0.3378 & 0.4141 & 0.3921 & 0.1604 & \cellcolor[gray]{0.8}{0.3479} & 0.1034 & \cellcolor[gray]{0.8}{0.1053} & 0.3246 & \cellcolor[gray]{0.8}{0.3551}       \\
                    & AUC    & 0.9688 & \cellcolor[gray]{0.8}{\textbf{\color{blue}{0.9698}}} & 0.9588 & 0.9586 & 0.9556 & \cellcolor[gray]{0.8}{0.9593} & 0.9308 & \cellcolor[gray]{0.8}{0.9498} & 0.9139 & \cellcolor[gray]{0.8}{0.9160} & 0.9605 & 0.9565        \\
        \midrule
        \multirow{3}{*}{Photo} & Hits@20 & 0.1056 & \cellcolor[gray]{0.8}{0.1160} & 0.0509 & \cellcolor[gray]{0.8}{0.0642} & 0.0753 & \cellcolor[gray]{0.8}{0.0946} & 0.0604 & \cellcolor[gray]{0.8}{\textbf{\color{blue}{0.1290}}} & 0.0348 & 0.0304 & 0.0304 & \cellcolor[gray]{0.8}{0.0655}       \\
                    & Hits@30 & 0.1200 & \cellcolor[gray]{0.8}{0.1389} & 0.0689 & \cellcolor[gray]{0.8}{0.0798} & 0.1019 & \cellcolor[gray]{0.8}{0.1208} & 0.0721 & \cellcolor[gray]{0.8}{\textbf{\color{blue}{0.1609}}} & 0.0457 & 0.0395 & 0.0386 & \cellcolor[gray]{0.8}{0.0885}        \\
                    & AUC    & 0.9353 & \cellcolor[gray]{0.8}{\textbf{\color{blue}{0.9437}}} & 0.8507 & \cellcolor[gray]{0.8}{0.8541} & 0.8760 & \cellcolor[gray]{0.8}{0.9165} & 0.9010 & \cellcolor[gray]{0.8}{0.9112} & 0.8136 & 0.8120 & 0.8124 & \cellcolor[gray]{0.8}{0.8599}        \\
        \midrule
        \multirow{3}{*}{Facebook} & Hits@20 & 0.2216 & \cellcolor[gray]{0.8}{0.2453} & 0.0987 & \cellcolor[gray]{0.8}{0.1093} & 0.2432 & \cellcolor[gray]{0.8}{\textbf{\color{blue}{0.2511}}} & 0.0619 & \cellcolor[gray]{0.8}{0.2188} & 0.0588 & \cellcolor[gray]{0.8}{0.1066} & 0.2222 & \cellcolor[gray]{0.8}{0.2411}       \\
                    & Hits@30 & 0.2902 & \cellcolor[gray]{0.8}{\textbf{\color{blue}{0.3389}}} & 0.1595 & \cellcolor[gray]{0.8}{0.2016} & 0.3034 & \cellcolor[gray]{0.8}{0.3046} & 0.0888 & \cellcolor[gray]{0.8}{0.2873} & 0.0848 & \cellcolor[gray]{0.8}{0.1518} & 0.2997 & 0.2842        \\
                    & AUC    & 0.9866 & \cellcolor[gray]{0.8}{\textbf{\color{blue}{0.9894}}} & 0.9783 & \cellcolor[gray]{0.8}{0.9786} & 0.9797 & \cellcolor[gray]{0.8}{0.9813} & 0.9386 & \cellcolor[gray]{0.8}{0.9529} & 0.9352 & \cellcolor[gray]{0.8}{0.9423} & 0.9802 & 0.9706         \\
        \midrule
        \multirow{3}{*}{OGB-DDI} & Hits@20 & 0.4546 & \cellcolor[gray]{0.8}{\textbf{\color{blue}{0.5227}}} & 0.1714 & \cellcolor[gray]{0.8}{0.2333} & 0.2179 & \cellcolor[gray]{0.8}{0.2386} & 0.1327 & \cellcolor[gray]{0.8}{0.3700} & 0.3489 & \cellcolor[gray]{0.8}{0.4789} & 0.1548 & \cellcolor[gray]{0.8}{0.2214}      \\
                    & Hits@30 & 0.5420 & \cellcolor[gray]{0.8}{\textbf{\color{blue}{0.6222}}} & 0.2157 & \cellcolor[gray]{0.8}{0.2713} & 0.2619 & \cellcolor[gray]{0.8}{0.2850} & 0.1327 & \cellcolor[gray]{0.8}{0.5527} & 0.4211 & \cellcolor[gray]{0.8}{0.5257} & 0.2301 & \cellcolor[gray]{0.8}{0.2755}        \\
                    & AUC    & 0.9965 & \cellcolor[gray]{0.8}{\textbf{\color{blue}{0.9968}}} & 0.9910 & \cellcolor[gray]{0.8}{0.9911} & 0.9929 & \cellcolor[gray]{0.8}{0.9931} & 0.9941 & \cellcolor[gray]{0.8}{0.9950} & 0.9961 & \cellcolor[gray]{0.8}{0.9962} & 0.9912 & \cellcolor[gray]{0.8}{0.9918 }       \\
        
    \bottomrule
    \end{tabular}}
    \label{tab:main_expert}
\end{table*}

\paratitle{Datasets and Evaluation Metrics}
We evaluate the  proposed MeBNS on six benchmark datasets, including citation graphs~\footnote{https://github.com/kimiyoung/planetoid/raw/master/data} (i.e.,  \textbf{Cora, CiteSeer, Pubmed} \cite{coracitepb}),  co-purchase graphs~\footnote{https://github.com/shchur/gnn-benchmark
/blob/master/data/npz/amazon\_electro nics\_photo .npz} (i.e., \textbf{Amazon Photo}~\cite{photo} ), social graphs~\footnote{https://docs.google.com/uc?export=download\&id
=12aJWAGCM4IvdGI2fiydDNy WzViEOLZH8 \&confirm=t} (i.e.,  \textbf{Facebook}~\cite{facebook}) and drug-drug interaction networks (i.e., \textbf{OGB-DDI}~\cite{hu2020open}), which are publicly available.  The details of the datasets are present in Table~\ref{tab:dataset}.
Following normal settings~\cite{yang2020understanding}, we randomly split  the entire edges of each dataset 
into training (70\%), validation (10\%) and test set (20\%), and especially for OGB-DDI, we used the OGB official train/validation/test splits.
Following the evaluation of link prediction tasks in OGB \cite{hu2020open}, we adopt the widely-used Hits@K ($K = 20, 30$ in our paper) and AUC as the evaluation metrics.



\paratitle{Baselines}
Since MeBNS is a general plugin, we perform experiments on several  link predictors with different negative samplers.

\begin{itemize}[leftmargin=*]
    \item \textbf{Link Predictor}: node-centric link predictors (i.e., GCN~\cite{gcn}, GAT~\cite{gat}) and edge-centric link predictors (i.e., SEAL~\cite{seal} and GraIL~\cite{grail}).
    \item \textbf{Negative Sampler}: static samplers (i.e., Uniform~\cite{uniform1}, PNS~\cite{caselles2018word2vec}), dynamic negative samplers (i.e., DNS~\cite{dns1}), graph based samplers (i.e., SANS~\cite{sans}, RecNS~\cite{recns}, MCNS~\cite{yang2020understanding})
\end{itemize}


%
%

\begin{table}[htbp]  
    \caption{Performance comparison w.r.t. Hits@20 with different backbones under Uniform and MCNS sampler. `-' means the original sampler while `+' means its variant improved by the MeBNS.}
    \centering
    \setlength{\tabcolsep}{1.0mm}{
    \begin{tabular}{c|c|c|c|c||c|c|c}
    \midrule
    \multicolumn{2}{c|}{} &\multicolumn{3}{c||}{Uniform sampler} &  \multicolumn{3}{c}{MCNS sampler} \\
    \midrule
    \multicolumn{2}{c|}{Backbone} & GAT & SEAL &GraIL  & GAT & SEAL &GraIL \\
    \midrule
    \multirow{2}{*}{Cora} & -  & 0.2928 & 0.5563 & 0.6009  & 0.5819 & 0.5260 & 0.2864 \\
       & +  &  \cellcolor[gray]{0.8}{0.5298} &  \cellcolor[gray]{0.8}{0.5696} & \cellcolor[gray]{0.8}{0.7061}  &  \cellcolor[gray]{0.8}{0.6123} &  \cellcolor[gray]{0.8}{0.5270} &  \cellcolor[gray]{0.8}{0.2929}\\
       \midrule
       \multirow{2}{*}{CiteSeer} & - & 0.4615 & 0.5879 & 0.5263 & 0.5439 & 0.5483 & 0.1223 \\
       & +  & \cellcolor[gray]{0.8}{0.4890} & \cellcolor[gray]{0.8}{0.5978} & \cellcolor[gray]{0.8}{0.6599} & \cellcolor[gray]{0.8}{0.6010} & \cellcolor[gray]{0.8}{0.6274} & \cellcolor[gray]{0.8}{0.1480} \\
       \midrule
       \multirow{2}{*}{Pubmed} & - & 0.1149 & 0.4123 & 0.4641 & 0.1296 & 0.4011 & 0.1686  \\
       & +  & \cellcolor[gray]{0.8}{0.1701} & \cellcolor[gray]{0.8}{0.4615} & \cellcolor[gray]{0.8}{0.5203}  & {0.0927} & \cellcolor[gray]{0.8}{0.4350} & \cellcolor[gray]{0.8}{0.3967}\\
    \midrule
    \end{tabular}}
    \label{tab:bcakbone}
\end{table}

\begin{figure}[t]
        \centering
        \subfigure[Uniform sampler]{\includegraphics[width=1.6in]{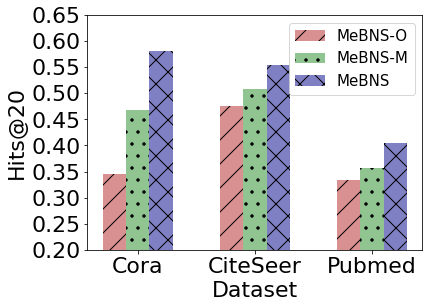}}
        \subfigure[MCNS sampler]{\includegraphics[width=1.6in]{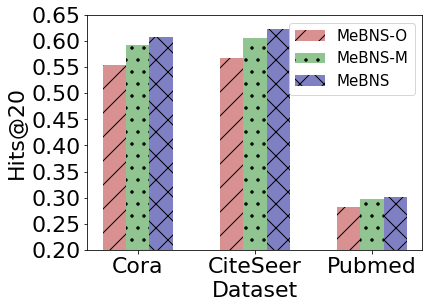}}
        \subfigure[Uniform sampler]{\includegraphics[width=1.6in]{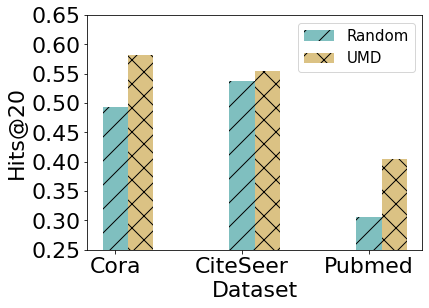}}
        \subfigure[MCNS sampler]{\includegraphics[width=1.6in]{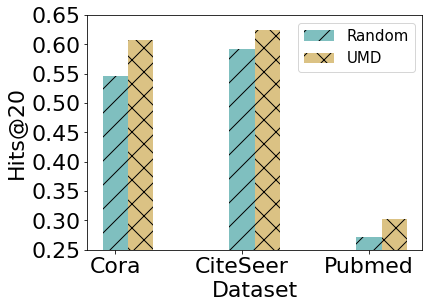}}
        \caption{(a), (b): Ablation studies of MeBNS. (c), (d): Impact of meta data collector.}
        \label{fig:xiaorong}
\end{figure}

\paratitle{Parameter Settings} 
We perform hyperparameters tuning on the validation set with the widely-used grid search strategy, and the optimal parameters as well as the experimental environment on different datasets are detailed in the \ref{sec_environ} and \ref{sec_ap_baseline}. Also, a detailed analysis of hyperparameters is presented in ~\ref{sec_ap_param}. Finally, we provide the source codes and demo datasets as the supplementary material for reproducibility.


\begin{figure*}[t]
        \centering
        \subfigure[AUC and loss with MeBNS ]{\includegraphics[width=2.2in]{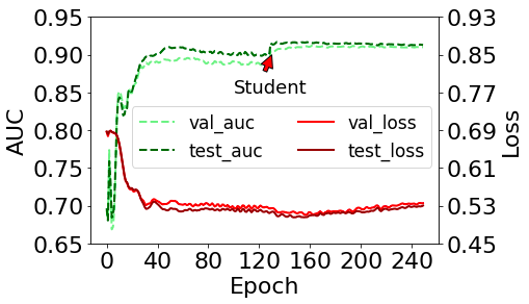}}
        \subfigure[Meta weight analysis]{\includegraphics[width=2.2in]{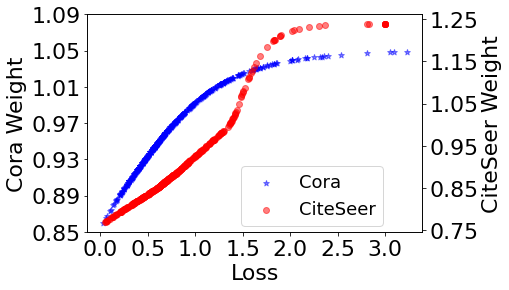}
        \label{fig:cora_cs_weight}}
        \subfigure[Migration of MeBNS]{\includegraphics[width=1.8in]{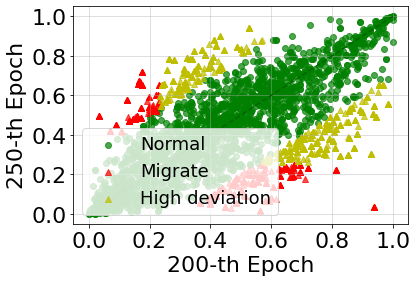}} 
        \caption{Qualitative analysis of MeBNS. \label{fig_qua}}
\end{figure*}

\subsection{Overall Performance}
With the GCN as the link prediction backbone, we report the experimental results w.r.t. Hits@K of MeBNS and baseline negative samplers on six datasets in Table~\ref{tab:main_expert} (Performance comparison with the AUC metric is reported in \ref{sec_auc_metric}.). Moreover, we also present the comparison results under different link prediction backbones in Table~\ref{tab:bcakbone}. The major observations are summarized as follows:
\begin{itemize}[leftmargin=*]
    \item \textbf{Outstanding Performance} MeBNS  consistently and significantly outperform all baselines on all datasets, achieves performance gains over the best baseline by  9.75\%, 10.07\%, 34.62\%, 22.15\%, 3.24\% and 14.98\% on the six datasets w.r.t. the metric Hit\@20 respectively. The results indicate the effectiveness of MeBNS for exploring and exploiting informative negatives.

    \item \textbf{General Plugin for Negative Samplers} In most cases, we could observe promising performance gains when integrating MeBNS with vanilla negative samplers. The underlying reason lies in that MeBNS overcome the bottleneck 
  of current negative samplers by alleviating the migration phenomenon and specializing in  ``hard" samples.

    \item \textbf{ Agnostic to Link Prediction Backbones} Towards existing mainstream node- (i.e., GCN and GAT) and edge-centric (i.e., SEAL and GraIL) backbones, MeBNS could be naturally adapted and achieves consistent performance improvements on three datasets. The model-agnostic property provides excellent flexibility for MeBNS being  adopted in various applications related to link prediction.  
\end{itemize}

\subsection{In-depth Analysis}
In this section, we perform a series of in-depth analyses to better understand the merits of MeBNS. 


\paratitle{Ablation Studies of MeBNS}
As the heart of the MeBNS, we take a closer look into the MST-GNN, consisting of a teacher-student design and a meta learner. 
Hence, we prepare two variants of MeBNS, namely \textbf{MeBNS-M} (i.e., MeBNS without the meta learner) and \textbf{MeBNS-O} (i.e., MeBNS without the  teacher-student design).
As reported in Fig.~\ref{fig:xiaorong} (a) \& (b), we find the overall performance is MeBNS $>$ MeBNS-M $>$ MeBNS-O. The observation reveals that MeBNS greatly benefits from our well-designed parts in MST-GNN: 
i) the teacher-student design helps alleviate the migration phenomenon (i.e., MeBNS-M $>$ MeBNS-O).
ii) the meta learner further encourages the model to distinguish ``hard'' samples in a fine-grained manner (i.e., MeBNS $>$ MeBNS-M). 


\paratitle{Analysis of UMD-Collector} To verify the validity of UMD-Collector, we collect the meta data by randomly sampling from the training set. 
Fig.~\ref{fig:xiaorong} (c) \& (d) shows that compared with random sampling, uncertainty-based meta data can better guide the training process of the student model.

\paratitle{Qualitative Analysis of MeBNS}
To have a deep insight into MeBNS, we carefully design a qualitative experiment to provide explicit evidence for MeBNS in improving the capability of link prediction. 
As shown in Fig.~\ref{fig_qua} (a), compared to the original learning curve in Fig.~\ref{fig:motivation_anly} (a), we find the loss (AUC) indeed  decreases (increases) after Epoch 100, attributed to that MeBNS alleviates the migration phenomenon with a better optimization landscape. 
On the other hand, the positive correlation between loss and meta weights in Fig.~\ref{fig_qua} (b) gives strong evidence that our meta leaner could help MeBNS locate ``hard'' samples in an adaptive way.
Further, through the experimental results in Fig.~\ref{fig_qua} (c), we observe that  MeBNS can alleviate the migration problem in DNS as the migration samples (i.e., ``\textcolor{red}{\UParrow}'') are significantly decreased. 

%% file: 2_relate_work.tex
\section{Related Work \label{sec_ap_rel}}

\paratitle{Link Prediction}, 
Link Prediction has been deeply explored and achieves excellent performance in various scenarios. \cite{jeh2002simrank} measure the similarity of links by defining heuristic features, \cite{perozzi2014deepwalk,grover2016node2vec} present the similarity by getting latent representations and \cite{ying2018graph,graphsage,seal} learning the way through GNN methods. 
Due to the powerful expressive ability of GNN, \cite{ying2018graph,graphsage,seal} reaches the SOTA. Among them
can be roughly divided into node-centric and edge-centric methods according to whether using subgraph features. Node centric methods \cite{GCNbased,graphsage} obtain representations by modeling their respective neighborhood and self information. Unfortunately, these methods only independently characterize the target nodes and ignore the structural information in subgraph which always maintain rich interactive information. Seal \cite{seal}  extracts subgraph information by \textit{Node Labeling}. Grail \cite{grail} mainly uses the \textit{Readout} function to represent the subgraph. PaGNN \cite{pagnn} and NBFNet \cite{bellman-ford} methods automatically learn subgraph information through broadcast and aggregation operations. BUDDY \cite{buddy} represents with subgraph sketching and Neo-GNNs \cite{neo} identifies by modeling the information of common neighbors. CFLP \cite{cflp} investigates the impact of structural information from a causal perspective and deals with counterfactual samples.

\paratitle{Negative Sampling}, 
There are sundry negative samplers like heuristics based \cite{uniform1,pns},  distribution based \cite{dns1,yang2020understanding}, GAN based \cite{wang2017irgan,cai2017kbgan,wang1711graphgan} and synthesis based \cite{huang2021mixgcf,mixmethod} etc., and have achieved surprising results. Here, we mainly focus on distribution methods that have attracted attention in recent works as other methods are unstable and resource-consuming like GAN-based. Typically, distribution methods can be divided into two categories, namely static sampler and dynamic sampler in which the probability of each sample being sampled in static samplers will not vary with the iteration of the model. 
Static samplers like \cite{rendle2014bayesian,pns,sans} choose samples from a uniform distribution or heuristic indicator (i.e. popularity, degrees, k-hops). Dynamic samplers continuously modify the sampling strategy as the model iterating. \cite{dns1} samples negatives with highest scores, \cite{yang2020understanding} conclude the optimal negative sampling distribution should be positively but sub-linearly correlated to positive sampling distribution. \cite{recns} proposes \textit{Three Region} strategy based on graph and uses positive samples to aid learning.

%% file: 6_appendix.tex
\appendix
\section{Appendix}

\subsection{Experiment Environments \label{sec_environ}}
All experiments are conducted with the following settings:
\begin{itemize}[leftmargin=*]
    \item Operating system: CentOS Linux release 7.6
    \item CPU: Intel(R) Xeon(R) CPU E5-2682 v4 @ 2.50GHz 
    \item GPU: 1*Tesla P100
    \item Software version: Python:3.6.8; Torch:1.8.1+cu101; Torch-geometric:2.0.3; Numpy:1.19.5; Scikit-learn:0.24.2; Scipy:1.5.4
\end{itemize}



\subsection{Baseline Settings \label{sec_ap_baseline}}

\paratitle{Backbone Settings}
\label{code source}

The backbones used in MeBNS are mainly based on the public implementations in PyG which are available in the following URLs, and the overall model architecture based on the node-centric (GCN, GAT)  and edge-centric (SEAL) GNN are presented in Algorithm \ref{alg:gcn_implement}  and Algorithm \ref{alg:seal_implement}, respectively. Normally, we set the hidden channels to 128 and the output channels to 64 in GCN, the hidden channels to 128, the output channels to 64 and heads to 5 in GAT, the hidden channels to 32 and the number of layers to 3 in SEAL. 

(1) GCN/GAT: https://github.com/pyg-team/pytorch\_geometric \\ /blob/master/examples/link\_pred.py

(2) SEAL: https://github.com/pyg-team/pytorch\_geometric/blob \\ /master/examples/seal\_link\_pred.py

\begin{algorithm}[htbp]
    \caption{GCN/GAT Based Implementation with PyTorch-style Pseudocode}
    \label{alg:gcn_implement}
    \textbf{Input}: Graph $\mathcal{G}=(\mathcal{V},\mathcal{E})$, Node feature matrix $X$ 
    \begin{algorithmic}[1]
        \STATE /** \textbf{Encoder} **/ \\
        \STATE \# Replacing GCNConv with GATConv when using GAT as Backbone
        \STATE Z = pyg$.$nn$.$GCNConv(in\_channels,hidden\_channels)($\mathcal{G}$,X)
        \STATE Z = pyg$.$nn$.$GCNConv(hidden\_channels,out\_channels)($\mathcal{G}$,Z)  \\
        \STATE /** \textbf{Decoder} **/ \\
        \FOR{(u,v) in Training Samples} 
        \STATE $Out_{u,v}= z_u \cdot z_v $\\
        \STATE Loss += BCELoss($Out_{u,v},Y_{u,v}$) 
        \ENDFOR
    \end{algorithmic}
\end{algorithm}

\begin{algorithm}[htbp]
    \caption{SEAL Implementation with PyTorch-style Pseudocode}
    \label{alg:seal_implement}
    \textbf{Input}: Subgraph $\mathcal{G}_{U,V}=(\mathcal{V},\mathcal{E})$, Node labeling matrix $X_{nl}$ 
    \begin{algorithmic}[1]
        \STATE /** \textbf{SEAL Encoder} **/
        \STATE Z = [$X_{nl}$]
        \FOR{conv in GCNConv$*$Num\_layers}
            \STATE Z += [conv(Z[-1], edge\_index)$_.$tanh()]
        \ENDFOR
        \STATE Z = torch$_.$cat(Z[1:], dim=-1)
        \STATE Z = GlobalSortPool(Z)
        \STATE Z = Conv1d(Z)$_.$relu()
        \STATE Z = MaxPool1d(Z)
        \STATE Z = Conv1d(Z)$_.$relu()
        \STATE /** \textbf{SEAL Decoder} **/
        \STATE Out = MLP(Z)
        \STATE Loss = BCELoss($Out,Y$) 
    \end{algorithmic}
\end{algorithm}

\paratitle{Hyperparameters Settings }

For MeBNS, we use the GNN architectures that are similar to the above settings in Sec. \ref{code source}, and we adopt an Adam optimizer in MST-GNN with setting learning rate to 0.01 while an SGD optimizer with learning rate 0.01. All hyperparameters settings are showed in Table~\ref{tab:hyper-params}:

\begin{table}[htbp]
    \caption{Hype-parameters settings. $T_{node}$ and $T_{edge}$ denote the number of iterations of the node-centric (GCN, GAT) and edge-centric (SEAL) teacher GNN, respectively.}
    \centering
    \begin{tabular}{l|rrrrr}
    
    \toprule
        Parameters & $T_{node}$ & $T_{edge}$ & $\beta$ & $\delta$ & $\tau$ \\
        \hline
        Cora & 100 &15 & 50\% & 5\% & $2*10^{-5}$  \\
        \hline
        CiteSeer & 200 &15 & 50\% & 30\% & $8*10^{-6}$  \\
        \hline
        Pubmed & 100 &15 & 50\% & 7\% & $8*10^{-7}$ \\
        \hline
        Amazon Photo & 200 &15 & 50\% & 3\% & $5*10^{-6}$ \\
        \hline
        Facebook & 200 &15 & 50\% & 1\% & $2*10^{-6}$ \\
        \hline
        OGB-DDI & 200 &10 & 10\% & 1\% & $3*10^{-6}$ \\
    \bottomrule
    \end{tabular}
    \label{tab:hyper-params}
\end{table}


More specifically, we present the probability of each sample being sampled as $p(u|v) \propto \text{Deg}(v)^{0.75}$ for PNS, which is commonly used in previous works \cite{yang2020understanding}. 
For SANS and RecNS, we set the subgraph hop $K=3$. In particular, unlike the recommendation scenario, there is no exposure information in link prediction datasets, thus we just replace RecNS with RecNS-O.

\subsection{Performance Comparison with the AUC Metric \label{sec_auc_metric}}
As AUC is an important evaluation metric in link prediction tasks, we also present the performance comparison w.r.t AUC in Table~\ref{tab:auc_expert} indicating that MeBNS gains on most  datasets with various negative sampler.

\subsection{Parameter Analysis  \label{sec_ap_param}}
Since MeBNS involves several important parameters, i.e., the number of iterations of teacher GNN $T$, the ratio of ``hard'' samples for student GNN $\beta$, the probability of graph based negative sampler $\delta$, we investigate their impacts in the following categories. We use the performance on Cora and CiteSeer as a reference, and similar patterns have been observed on the other datasets.

Firstly, we vary $T$ in the set of \{50, 100, 150, 200, 250, 300\}. As shown in Fig.~\ref{fig:split_time}, MeBNS achieves the optimal performance when the teacher GNN converges to a relatively stable state (i.e., $T = 100$ in Cora and 200 in CiteSeer). A small $T$ or a large $T$ would cause the under-fitting or over-fitting of the teacher GNN, resulting in unpromising  performance.

In terms of the parameter $\beta$, we present the results in Fig.~\ref{fig:split_ratio}. We find that MeBNS achieves the optimal performance near 50\% on both datasets and too small $\beta$ would harm the model by potentially filtering a number of informative samples.

Lastly, we run a grid search over \{0.0, 0.05, 0.1, 0.3, 0.5\} for $\delta$ and plot their performance in Fig.~\ref{fig:structure_ratio}. The optimal performance is yielded near 0.05 in Cora and 0.3 in CiteSeer showing the usefulness of graph-based negatives, although too large $\delta$ would hinder the capability of MeBNS due to the excessive involvement of false negatives. 

\begin{figure}[htbp]
        \centering
        \subfigure[Cora]{\includegraphics[width=1.6in]{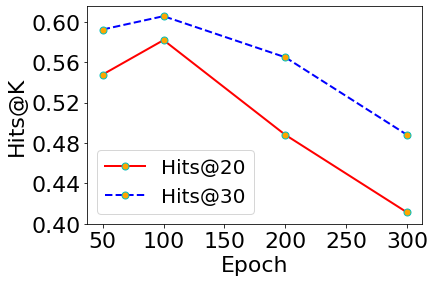}
        \label{fig:cora_split_time}}
        \subfigure[CiteSeer]{\includegraphics[width=1.6in]{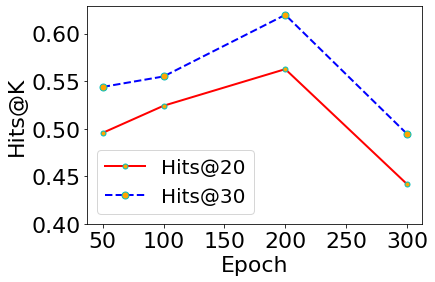}
        \label{fig:cs_split_time}}
        \caption{Impact of the number of iterations of teacher GNN.}
        \label{fig:split_time}
\end{figure}

\begin{figure}[htbp]
        \centering
        \subfigure[Cora]{\includegraphics[width=1.6in]{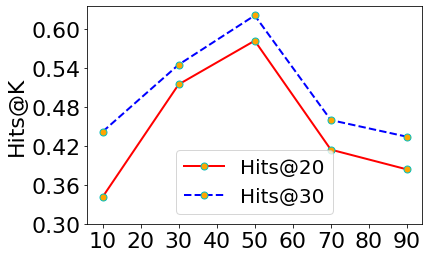}
        \label{fig:cora_split_ratio}}
        \subfigure[CiteSeer]{\includegraphics[width=1.6in]{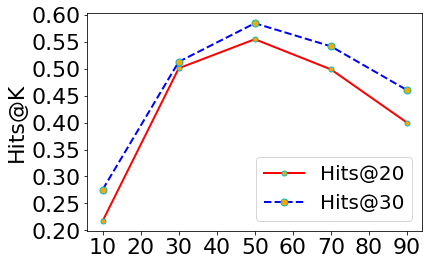}
        \label{fig:cs_split_ratio}}
        \caption{Impact of the  ratio of hard samples(\%).}
        \label{fig:split_ratio}
\end{figure}

\begin{figure}[htbp]
        \centering
        \subfigure[Cora]{\includegraphics[width=1.6in]{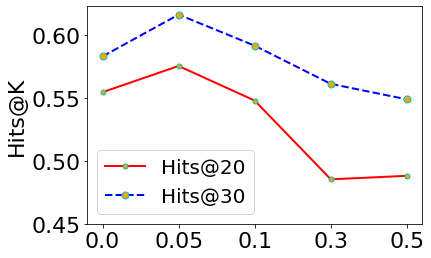}
        \label{fig:cora_structure}}
        \subfigure[CiteSeer]{\includegraphics[width=1.6in]{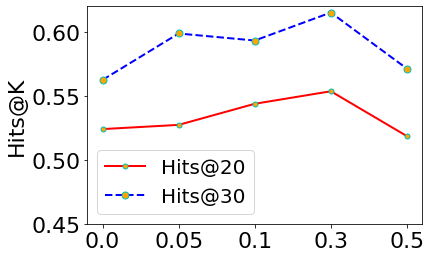}
        \label{fig:cs_structure}}
        \caption{Impact of the probability of graph based negative sampler.}
        \label{fig:structure_ratio}
\end{figure}

\subsection{Qualitative Analysis on CiteSeer \label{sec_ap_qua}}
We also conduct qualitative analysis on the CiteSeer dataset in  Fig. \ref{fig:cs_motivation_anly} and have similar findings with Sec.~\ref{analysis} that the model converges to a deeper landscape and alleviates migration with MeBNS. 
\begin{figure}[htbp]
        \centering
        \subfigure[AUC and Loss with DNS ]{\includegraphics[width=1.4in]{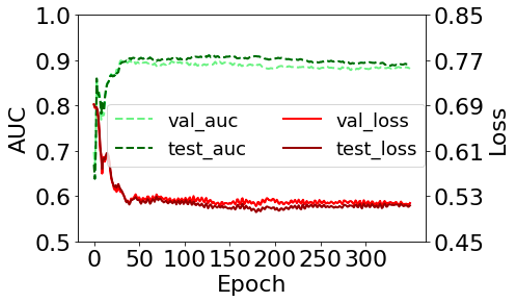}}
        \subfigure[Migration of DNS]{\includegraphics[width=1.2in]{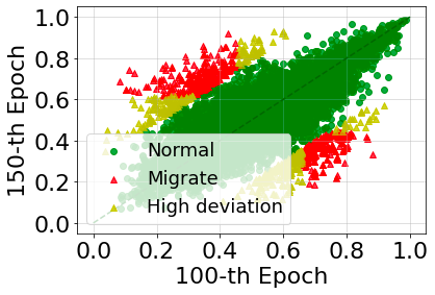}} 
        \subfigure[AUC and Loss with MeBNS ]{\includegraphics[width=1.45in]{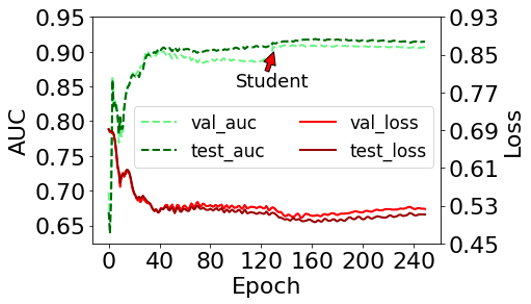}}
        \subfigure[Migration of MeBNS]{\includegraphics[width=1.2in]{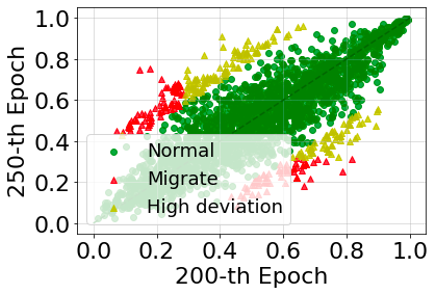}} 
        \caption{Qualitative analysis on CiteSeer.}
        \label{fig:cs_motivation_anly}
\end{figure}